\documentclass{article}

\usepackage[preprint]{neurips_2026}

\usepackage[utf8]{inputenc} %
\usepackage[T1]{fontenc}    %
\usepackage{hyperref}       %
\usepackage{url}            %
\usepackage{booktabs}       %
\usepackage{amsfonts}       %
\usepackage{nicefrac}       %
\usepackage{microtype}      %
\usepackage{amsmath}
\usepackage{multirow}
\usepackage{amssymb}
\usepackage{bm}
\newcommand{\ourmodel}{TraFL}
\usepackage{graphicx} 
\usepackage{amsmath}
\usepackage{amsthm}
\usepackage[utf8]{inputenc} %
\usepackage[T1]{fontenc}    %
\usepackage{hyperref}       %
\usepackage{url}            %
\usepackage{booktabs}       %
\usepackage{amsfonts}       %
\usepackage{nicefrac}       %
\usepackage{microtype}      %
\usepackage{adjustbox}
\usepackage[table,xcdraw]{xcolor}
\usepackage{comment}
\usepackage{arydshln}
\usepackage{tcolorbox}
\tcbuselibrary{skins,breakable}
\usepackage{xspace}
\usepackage{float}
\usepackage{fontawesome5}
\usepackage{makecell} 
\usepackage{todonotes}
\usepackage{tabularx} %
\usepackage{pifont}       %
\usepackage{enumitem}     %

\usepackage{listings}
\usepackage{upquote}

\definecolor{lightblue}{rgb}{0.678, 0.847, 0.902}  %

\newtcolorbox{findingbox}{
  colframe=black!20,         %
  colback=gray!5,            %
  boxrule=0.4pt,             %
  arc=1mm,                   %
  left=4pt, right=4pt, top=4pt, bottom=4pt, %
  breakable,
  enhanced,
  before skip=8pt,
  after skip=8pt,
}

\usepackage[capitalize,nameinlink,noabbrev]{cleveref}
\usepackage{float}
\usepackage[ruled,vlined]{algorithm2e}
\Crefname{figure}{Fig.}{Figs.}
\Crefname{table}{Tab.}{Tabs.}
\Crefname{section}{Sec.}{Secs.}
\Crefname{appendix}{App.}{Apps.}
\Crefname{equation}{Eq.}{Eqs.}

\definecolor{prompt1}{RGB}{223, 223, 192}
\definecolor{prompt1-frame}{RGB}{137, 137, 90}
\definecolor{prompt2}{RGB}{180, 230, 210}
\definecolor{prompt2-frame}{RGB}{70, 140, 120}
\definecolor{prompt3}{RGB}{212, 238, 179}
\definecolor{prompt3-frame}{RGB}{117, 146, 77}

\tcbset{
    prompt1/.style={
        enhanced,
        breakable,
        colback=prompt1,        %
        colframe=prompt1-frame,           %
        fontupper=\ttfamily,      %
        boxrule=1pt,              %
        arc=4mm,                  %
        left=1mm, right=1mm, top=1mm, bottom=1mm, %
        boxsep=4pt,               %
        before skip=10pt, after skip=10pt, %
        overlay={
            \node[anchor=north west, xshift=4pt, text=white] at (frame.north west) {\faDatabase};
        },
        title={~~~~~~~\textbf{Prompt for In-context Learning of Chain-of-thought Editing}},
        coltitle=white,
        fonttitle=\bfseries\small
    }
}

\tcbset{
    prompt2/.style={
        enhanced,
        breakable,
        colback=prompt2,        %
        colframe=prompt2-frame,            %
        fontupper=\ttfamily,               %
        boxrule=1pt,                       %
        arc=4mm,                           %
        left=1mm, right=1mm, top=1mm, bottom=1mm, %
        boxsep=4pt,                        %
        before skip=10pt, after skip=10pt, %
        overlay={
            \node[anchor=north west, xshift=4pt, text=white] at (frame.north west) {\faDatabase};
        },
        title={~~~~~~~\textbf{Chain-of-thought Reasoning Examples (VisMin dataset)}},
        coltitle=white,
        fonttitle=\bfseries\small
    }
}

\tcbset{
    prompt4/.style={
        enhanced,
        breakable,
        colback=prompt2,        %
        colframe=prompt2-frame,            %
        fontupper=\ttfamily,               %
        boxrule=1pt,                       %
        arc=4mm,                           %
        left=1mm, right=1mm, top=1mm, bottom=1mm, %
        boxsep=4pt,                        %
        before skip=10pt, after skip=10pt, %
        overlay={
            \node[anchor=north west, xshift=4pt, text=white] at (frame.north west) {\faDatabase};
        },
        title={~~~~~~~\textbf{Chain-of-thought Reasoning Examples}},
        coltitle=white,
        fonttitle=\bfseries\small
    }
}

\tcbset{
    prompt3/.style={
        enhanced,
        breakable,
        colback=prompt3,        %
        colframe=prompt3-frame,            %
        fontupper=\ttfamily,               %
        boxrule=1pt,                       %
        arc=4mm,                           %
        left=1mm, right=1mm, top=1mm, bottom=1mm, %
        boxsep=4pt,                        %
        before skip=10pt, after skip=10pt, %
        overlay={
            \node[anchor=north west, xshift=4pt, text=white] at (frame.north west) {\faDatabase};
        },
        title={~~~~~~~\textbf{Prompt for In-context Learning of Chain-of-thought Editing}},
        coltitle=white,
        fonttitle=\bfseries\small
    }
}

\definecolor{rolloutbg}{RGB}{245,247,250}
\definecolor{rolloutframe}{RGB}{70,95,130}
\definecolor{tealprompt}{RGB}{0,128,128}
\definecolor{purpleprompt}{RGB}{128,0,180}
\definecolor{reasoningcolor}{RGB}{180,90,0}
\definecolor{systemgray}{RGB}{90,90,90}
\definecolor{codebg}{RGB}{248,248,248}

\lstdefinestyle{mypython}{
  language=Python,
  basicstyle=\ttfamily\tiny,
  backgroundcolor=\color{codebg},
  frame=single,
  rulecolor=\color{black!20},
  showstringspaces=false,
  columns=fullflexible,
  keepspaces=true
}

\tcbset{
  code_prompt_rollout/.style={
    enhanced,
    breakable,
    colback=rolloutbg,
    colframe=rolloutframe,
    boxrule=0.8pt,
    arc=2mm,
    left=1mm, right=1mm, top=1mm, bottom=1mm,
    boxsep=4pt,
    before skip=10pt, after skip=10pt,
    colbacktitle=rolloutframe,
    coltitle=white,
    fonttitle=\bfseries\small,
    title={\textbf{Rollout Example: System Prompt and Sample Model Output}},
    attach boxed title to top left={xshift=0mm,yshift=0mm},
    boxed title style={
      sharp corners,
      size=small,
      boxrule=0pt,
      colframe=rolloutframe,
      colback=rolloutframe
    }
  }
}

\newtheorem{theorem}{Theorem}
\newtheorem{proposition}{Proposition}
\newtheorem{corollary}{Corollary}

\definecolor{theorem_bg}{RGB}{255,238,232}
\definecolor{theorem_frame}{RGB}{255,238,232}

\newtcolorbox{theorembox}{
  enhanced,
  breakable,
  colback=theorem_bg,
  colframe=theorem_frame,
  boxrule=0pt,
  arc=0pt,
  left=6pt,
  right=6pt,
  top=5pt,
  bottom=5pt,
  before skip=0.8em,
  after skip=0.8em
}

\definecolor{theorem_bg}{RGB}{255,238,232}

\newtcolorbox{repeatedtheorembox}{
  enhanced,
  breakable,
  colback=theorem_bg,
  colframe=theorem_bg,
  boxrule=0pt,
  arc=0pt,
  left=6pt,
  right=6pt,
  top=5pt,
  bottom=5pt,
  before skip=0.8em,
  after skip=0.8em
}

\usepackage{silence}
\title{
Beyond Mode-Seeking RL: Trajectory-Balance Post-Training for Diffusion Language Models}

\author{\bf Saba Ahmadi\thanks{denotes equal contribution}\hspace{1em} 
    Prasanna Parthasarathi\footnotemark[1]\hspace{1em}
    Yufei Cui \\
    Noah’s Ark Lab  \quad 
}

\begin{document}

\maketitle

\begin{abstract}
Diffusion language models are a promising alternative to autoregressive models, yet post-training methods for them largely adapt reward-maximizing objectives. We identify a central failure mode in this setting we call \emph{trajectory locking}: sampled reward-driven updates over-concentrate probability mass onto a narrow set of denoising paths, reducing coverage of alternative correct solutions under repeated sampling. To address this, we propose \ourmodel\ (\textbf{Tra}jectory \textbf{F}low ba\textbf{L}ancing), a trajectory-balance objective that trains the policy toward a reward-tilted target distribution anchored to a frozen reference model. We make this practical for diffusion language models with a diffusion-compatible sequence-level surrogate and a learned prompt-dependent normalization. Across mathematical reasoning and code generation benchmarks, \ourmodel\ is the only evaluated post-training method that improves over the base model in every benchmark--length setting, with gains that persist as the sampling budget increases. The improvements transfer to held-out evaluations: \ourmodel\ stays above the base model on Minerva Math~\citep{minerva} and is the strongest method on every LiveCodeBench~\citep{jain2025livecodebench} difficulty split. 
\end{abstract}
    
\section{Introduction}
\label{sec:intro}
Recent diffusion language models (dLLMs) have emerged as a compelling alternative to autoregressive models, showing early promise on reasoning and code generation~\citep{llada,ye2025dream}. A key open question for advancing these models further is how to post-train them effectively. From a capability perspective, current post-training methods improve single-sample accuracy but struggle to produce a \emph{diverse} set of correct solutions under repeated sampling---a key requirement when many valid answers exist.

Designing such an objective is challenging because diffusion language models do not expose token-level conditional log-probabilities in the way autoregressive models do. Existing approaches therefore adapt reward-maximizing RL to diffusion generation through surrogate likelihoods, autoregressive reductions, or stepwise approximations~\citep{zhao2025diffuGRPO,tang2025wd1,wang2025spg,kunde2026egspo,ni2026flexibility}. Across these methods, we observe a recurring failure mode we call \emph{trajectory locking}: because reward depends only on the final completion and is indifferent to which denoising path produced it, sampled policy-gradient updates reinforce already-favored paths, progressively collapsing probability mass onto a narrow subset of trajectories---and with it, coverage of alternative correct solutions. Distribution-matching methods that train toward a reward-tilted target~\citep{zhu2026flowrl,dmpo} offer a principled alternative, yet as we show, the way the normalization term is handled can reintroduce the same concentration failure. This raises a basic question: \emph{what is the right post-training objective for diffusion language models that avoids trajectory locking while remaining practical?}

To this end, we propose \ourmodel\ (\textbf{Tra}jectory \textbf{F}low ba\textbf{L}ancing), a post-training objective grounded in the trajectory-balance principle from Generative Flow Networks (GFlowNets)~\citep{bengio2026gflownetfoundations}. Rather than using reward as an unconstrained amplifier of sampled trajectories, \ourmodel\ trains the policy toward a reward-tilted target distribution anchored to a frozen reference model. In this target, the reward increases the mass assigned to successful completions, while the reference model regularizes how probability mass is allocated over plausible generations. We make this practical for diffusion language models with two ingredients: (i) a diffusion-compatible sequence-level surrogate for comparing fully denoised completions under the current and reference models, and (ii) a learned prompt-dependent normalization term, jointly trained with the policy, that receives gradient signal at every training step regardless of which completions are sampled.

Across mathematical reasoning (GSM8K~\citep{cobbe2021gsm8k}, MATH-500~\cite{lightman2023math500}) and code generation (HumanEval~\citep{chen2021humaneval}, MBPP~\citep{mbpp}) benchmarks, \ourmodel\ is the only directly evaluated post-training method that improves over the base model in every benchmark-length setting. The gains transfer to held-out evaluations on Minerva Math~\citep{minerva} and LiveCodeBench~\citep{jain2025livecodebench}, where \ourmodel\ is the strongest method on every LiveCodeBench difficulty split.

\noindent Our contributions are as follows:
\begin{enumerate}[leftmargin=2em, itemsep=2pt, parsep=2pt]
    \item We identify \emph{trajectory locking} as a central failure mode of dLLM post-training, where sampled reward-driven updates collapse probability mass onto a narrow set of denoising paths.
    \item We propose \ourmodel\ (\textbf{Tra}jectory \textbf{F}low ba\textbf{L}ancing), a trajectory-balance objective for post-training diffusion language models that trains toward a reference-anchored reward-tilted target distribution, made practical with a diffusion-compatible sequence-level surrogate and a learned prompt-dependent normalization.
    \item \ourmodel\ is the only evaluated post-training method that improves over the base model in every benchmark-length setting on GSM8K, MATH-500, HumanEval, and MBPP, outperforming \textsc{ESPO}~\citep{ou2026espo_principled} and \textsc{JustGRPO}~\citep{ni2026flexibility} on average.
    \item The gains transfer to held-out Minerva Math~\citep{minerva} and LiveCodeBench~\citep{jain2025livecodebench}, where \ourmodel\ is the strongest method on every LiveCodeBench difficulty split, and an LLM-as-judge analysis provides evidence that improvements are associated with broader correct-solution coverage rather than only sharper scoring of a single mode.
   
\end{enumerate}

\section{Related Work}
\label{sec:related_work}

\paragraph{RL post-training for diffusion language models.}
Recent diffusion language models such as LLaDA~\citep{llada} and Dream~\citep{ye2025dream} have made diffusion-based text generation a viable alternative to autoregressive language models, but reinforcement learning for these models remains challenging because diffusion generation does not expose the same left-to-right token-level conditional factorization used by PPO- or GRPO-style training in autoregressive LLMs. Early work adapted autoregressive RL objectives to diffusion models by introducing likelihood surrogates. In particular, \textsc{d1}~\citep{zhao2025diffuGRPO} proposes \textsc{diffu-GRPO}, a GRPO-style method for masked dLLMs built on one-step per-token log-probability estimation together with a mean-field approximation to sequence likelihood, and combines it with a preceding supervised fine-tuning stage in its full recipe. \textsc{wd1}~\citep{tang2025wd1} removes explicit policy-ratio estimation and instead optimizes a ratio-free weighted log-likelihood objective derived from reverse-KL-regularized policy optimization. \textsc{SPG}~\citep{wang2025spg} addresses the bias induced by one-sided likelihood surrogates by maximizing a lower bound for positive-advantage samples and minimizing an evidence upper bound for negative-advantage samples, together with a block-wise masking strategy for more stable Monte Carlo estimation.

A complementary line of work argues that the central issue is not only the quality of the surrogate, but also the action granularity used by the RL objective. \textsc{ESPO}~\citep{ou2026espo_principled} formalizes this view most explicitly: it treats whole-sequence generation as a single action and uses the ELBO as a tractable sequence-level proxy, together with stabilized ratio and KL estimators, arguing that token-level objectives are fundamentally mismatched to non-autoregressive diffusion generation. \textsc{TraceRL}~\citep{wang2025tracerl} takes yet another perspective, emphasizing alignment between the training objective and the model's preferred inference trajectory. It performs trajectory-aware optimization over denoising traces, and introduces a diffusion-based value model for variance reduction. In contrast, \textsc{JustGRPO}~\citep{ni2026flexibility} argues that preserving arbitrary-order generation during RL can itself be counterproductive for reasoning, and instead constrains training to autoregressive order so that standard GRPO can be applied directly, while still retaining the parallel decoding benefits of dLLMs at inference time.

\paragraph{Distribution matching and reward-tilted target distributions.}
Closest in spirit to our work are methods that move beyond pure reward maximization and instead optimize toward reward-tilted target distributions. On the dLLM side, \textsc{DMPO}~\citep{dmpo} formulates post-training as policy distribution matching toward a reward-tilted target distribution and implements this through importance sampling and weighted denoising cross-entropy, i.e., a scalable forward-KL-style approximation to distribution matching. On the autoregressive side, \textsc{FlowRL}~\citep{zhu2026flowrl} likewise advocates matching a reward-tilted distribution rather than maximizing reward alone, but does so through a GFlowNet-style formulation for autoregressive reasoning models rather than diffusion language models. Our method is aligned with this broader distribution-matching view, but differs in how it is instantiated for dLLMs: we adopt a trajectory-balance perspective tailored to diffusion generation, combine it with a diffusion-compatible sequence-level surrogate, and learn a prompt-dependent partition function that captures the normalization of the reward-tilted target.

\section{Trajectory Flow Balancing}
\label{sec:method}

We introduce \ourmodel\ (\textbf{Tra}jectory \textbf{F}low ba\textbf{L}ancing), a reference-anchored trajectory-balance objective for post-training diffusion language models. We first formalize why terminal rewards can induce trajectory locking, then define the reward-tilted target underlying \ourmodel.

\paragraph{Setup: Diffusion Trajectories and Terminal Rewards}
\label{sec:method_setup}

Given a prompt \(\mathbf{x}\), a discrete diffusion language model defines a distribution over denoising trajectories
\[
\tau^{(i)} = \left(z_T^{(i)}, z_{T-1}^{(i)}, \ldots, z_0^{(i)}\right),
\qquad
\tau^{(i)} \sim p_\theta(\tau \mid \mathbf{x}).
\]
Here, \(z_T\) denotes the fully noised state and \(z_0\) the final completion \(\mathbf{y}\). 
The reward \(r(\mathbf{x},\mathbf{y})\) is assigned only to this terminal completion, while the model distribution is induced through the full denoising trajectory. 
This distinction is important because several denoising trajectories can terminate in the same completion \(\mathbf{y}\), while different completions can correspond to different valid solution modes. Thus, post-training can affect both the allocation of probability over paths to a fixed answer and the coverage of distinct terminal solution modes.

\subsection{Trajectory Locking: Why Terminal Reward Is Not Enough}
\label{sec:trajectory_locking}

Maximizing terminal reward alone is blind to how probability mass is allocated across denoising paths that reach the same completion. 
The objective depends on the total probability assigned to a final completion \(\mathbf{y}\), but not on how that probability is split among the denoising trajectories that produce \(\mathbf{y}\). 
We formalize this path-indifference below.
\begin{theorembox}
\begin{proposition}[Path-indifference of terminal reward]
\label{prop:path_indifference}
Fix a prompt \(\mathbf{x}\) and a final completion \(\mathbf{y}\). 
Let \(\mathcal{T}(\mathbf{y})\) denote the set of denoising trajectories that terminate at \(\mathbf{y}\). 
Any redistribution of probability mass among trajectories in \(\mathcal{T}(\mathbf{y})\) that preserves \(p_\theta(\mathbf{y}\mid\mathbf{x})\) leaves a terminal-reward objective unchanged.
\end{proposition}
\end{theorembox}

\begin{proof}
The terminal-reward objective decomposes as
\[
J(\mathbf{x})
=
\sum_{\mathbf{y}} p_\theta(\mathbf{y}\mid \mathbf{x})\, r(\mathbf{x},\mathbf{y}),
\qquad
p_\theta(\mathbf{y}\mid \mathbf{x})
=
\sum_{\tau\in\mathcal{T}(\mathbf{y})}
p_\theta(\tau\mid \mathbf{x}).
\]
The contribution of \(\mathbf{y}\) to \(J\) depends only on the scalar \(p_\theta(\mathbf{y}\mid\mathbf{x})\), not on how that mass is spread across trajectories in \(\mathcal{T}(\mathbf{y})\). 
Therefore, any redistribution of mass within \(\mathcal{T}(\mathbf{y})\) that keeps \(p_\theta(\mathbf{y}\mid\mathbf{x})\) unchanged leaves \(J(\mathbf{x})\) unchanged.
\end{proof}

Proposition~\ref{prop:path_indifference} is an objective-level statement: reward maximization does not distinguish among paths that reach the same terminal completion. 
In sampled optimization, however, this flatness can become unstable. 
Only sampled trajectories receive direct gradient signal. 
Consequently, if two trajectories reach the same rewarding completion but one is sampled slightly more often early in training, it receives more positive updates, becomes more likely, and is sampled even more often in subsequent updates. 
We call this self-reinforcing concentration \emph{trajectory locking}. 
A formal analysis of this feedback effect is given in App.~\ref{app:trajectory_locking} and App.~\ref{app:replicator}.

Trajectory locking matters because trajectory diversity upper-bounds terminal solution coverage. 
We make this connection explicit next.

\begin{theorembox}
\begin{theorem}[Trajectory diversity is necessary for mode coverage]
\label{thm:trajectory_diversity}
Let \(\mathcal{T}\) be a random denoising trajectory taking values in a countable set \(\Omega_{\mathcal{T}}\), and let \(M=g(\mathcal{T})\) be its terminal solution mode under a deterministic mapping \(g: \Omega_{\mathcal{T}} \to \Omega_M\). Then
\[
|\mathrm{supp}(M)| \le |\mathrm{supp}(\mathcal{T})|
\qquad\text{and}\qquad
H(M)\le H(\mathcal{T}),
\]
where \(\mathrm{supp}(\cdot)\) denotes the support of a distribution and \(H(\cdot)\) the Shannon entropy. In particular, terminal mode coverage cannot exceed trajectory diversity.
\end{theorem}
\end{theorembox}

\begin{theorembox}
\begin{corollary}[Trajectory locking bounds mode coverage]
\label{cor:locking}
If trajectory locking reduces the trajectory distribution to a subset \(\mathcal{S}\) of all possible paths, then the set of terminal modes that can be covered is at most \(g(\mathcal{S})\). In particular, collapse to a single trajectory implies collapse to a single terminal mode.
\end{corollary}
\end{theorembox}
Proofs of Theorem~\ref{thm:trajectory_diversity} and Corollary~\ref{cor:locking} are given in App.~\ref{app:theory_mainproofs}. 
Together, these statements make the stakes concrete: trajectory locking does not merely change which paths are used; it can limit which solutions the model can produce under sampling. 
This motivates a reference-anchored objective that uses reward to tilt terminal completions, while preserving a non-degenerate allocation of probability mass over trajectories.

\subsection{Reward-Tilted Trajectory Flow}
\label{sec:reward_tilted_flow}

To avoid relying on terminal reward alone to determine trajectory allocation, we define a reward-tilted target distribution over denoising trajectories by reweighting a frozen reference model:
\begin{equation}
\label{eq:target_dist}
p^*(\tau \mid \mathbf{x})
=
\frac{1}{Z(\mathbf{x})}
p_{\mathrm{ref}}(\tau \mid \mathbf{x})
\exp\bigl(\beta r(\mathbf{x}, \mathbf{y})\bigr),
\end{equation}
where \(r(\mathbf{x}, \mathbf{y})\) is a scalar reward on the final completion, \(\beta > 0\) controls reward sharpness, and \(Z(\mathbf{x})\) is the prompt-dependent normalization term. 
Although the reward is still terminal, it does not define the target distribution by itself: it tilts a frozen reference trajectory distribution. Thus, among trajectories with the same terminal reward, the target remains proportional to \(p_{\mathrm{ref}}(\tau\mid\mathbf{x})\), rather than being indifferent to how mass is allocated across them. This separates two roles that are coupled in direct reward maximization: the reward controls how much mass high-reward completions should receive, while the reference model anchors the relative allocation of probability over plausible trajectories and completions, rather than letting reward amplify only the sampled support.

\paragraph{Trajectory Flow Balancing Objective}
\label{sec:trafl_objective}

We train the current model to match the reward-tilted target in \cref{eq:target_dist} through a trajectory-balance residual. 
Equating \(p_\theta(\tau\mid\mathbf{x})\) with \(p^*(\tau\mid\mathbf{x})\) and taking logarithms gives
\begin{equation}
\label{eq:tb_residual}
\delta(\tau, \mathbf{x})
=
\log p_\theta(\tau \mid \mathbf{x})
-
\log p_{\mathrm{ref}}(\tau \mid \mathbf{x})
-
\beta r(\mathbf{x}, \mathbf{y})
+
\log Z(\mathbf{x}),
\end{equation}
and we minimize the squared loss,
\begin{equation}
\label{eq:tb_loss}
\mathcal{L}_{\mathrm{TraFL}}
=
\mathbb{E}_{\tau \sim p_\theta(\cdot \mid \mathbf{x})}
\bigl[\delta(\tau, \mathbf{x})^2\bigr].
\end{equation}
At the minimum, the residual is zero when the current trajectory distribution matches the reward-tilted reference distribution. 
The learned normalization term \(Z(\mathbf{x})\) absorbs the prompt-dependent offset in the balance condition, so the policy is trained on relative trajectory-level preferences rather than on an unnormalized reward signal alone. 
Equations~\eqref{eq:target_dist}--\eqref{eq:tb_loss} define the trajectory-level balance condition.

\paragraph{Normalized Log-Probability Surrogate}
\label{sec:seq_surrogate}

Evaluating \cref{eq:tb_residual} requires log-probability terms under both the current and reference models. Diffusion language models do not expose exact trajectory log-probabilities, so we use a tractable log-probability surrogate derived from a masked-reconstruction lower bound. Given a completion \(\mathbf{y}=(y_1,\dots,y_L)\), we sample \(l \sim \mathrm{Uniform}(\{1,\dots,L\})\), construct a corrupted sequence \(\mathbf{y}_l \sim q_l(\cdot \mid \mathbf{y}, \mathbf{x})\) by replacing exactly \(l\) uniformly sampled completion tokens with \(\mathtt{[mask]}\), and define
\begin{equation}
\label{eq:seq_surrogate}
{\log \hat p}_\theta(\mathbf{y}\mid\mathbf{x})
=
\underset{l \sim \mathcal{U}(\{1,\dots,L\})}{\mathbb{E}}
\;
\underset{{\mathbf{y}_l \sim q_l(\cdot \mid \mathbf{y},\mathbf{x})}}{\mathbb{E}}
\left[
\frac{1}{l}
\sum_{i=1}^{L}
\mathbf{1}[y_i^l = \mathtt{[mask]}]
\log p_\theta(y_i \mid \mathbf{y}_l,\mathbf{x})
\right].
\end{equation}
The indicator picks out the \(l\) masked positions, so the bracketed quantity is the average log-probability across those \(l\) positions; the outer expectations are over the sampled mask level \(l\) and corruption \(\mathbf{y}_l\). The full derivation from a masked-reconstruction lower bound, and the role of the \(1/l\) factor, are given in App.~\ref{app:surrogate_elbo}. We estimate \cref{eq:seq_surrogate} with Monte Carlo samples and use the same masking pattern for the current and reference models to reduce variance.

\paragraph{Why Distribution-Matching Alone Is Insufficient: The Role of a Learned \texorpdfstring{$Z$}{Z}.}
\textsc{DMPO}~\citep{dmpo} also targets a reward-tilted distribution but does not learn \(Z(\mathbf{x})\); it instead estimates the partition function with a softmax over the current rollout buffer. This makes the normalization rollout-local: once sampled completions begin to concentrate, both the estimated target weights and the resulting update are computed over the same narrowed support, and completions outside that support receive no gradient signal. As we discuss in App.~\ref{app:dmpo_replicator}, this leads to a self-reinforcing concentration over sampled completions. 

\section{Experimental Setup}
\label{sec:exp_setup}

\subsection{Training Details}
\label{sec:training_details}

\paragraph{Implementation.}
Our base model is LLaDA-8B-Instruct~\citep{zhu2025lladaInstruct}, a masked diffusion language model that generates text by iterative denoising over discrete token sequences. For RL post-training, we apply \ourmodel\ directly to the instruction-tuned checkpoint, without an additional SFT stage. We use LoRA~\citep{hu2022lora} adapters with rank \(r=128\) and scaling factor \(\alpha=128\) for parameter-efficient fine-tuning. The partition-function head \(\log Z(\mathbf{x})\) is parameterized as a 2-layer MLP over mean-pooled prompt hidden states, with backbone features detached so that optimizing the partition head does not alter the backbone's prompt representations. We train with a peak learning rate of \(6\times10^{-5}\) under a cosine decay schedule with minimum learning rate \(6\times10^{-6}\), and generate 5 rollouts per prompt. The masked surrogate log-probability \({\log \hat p}_\theta(\mathbf{y}\mid\mathbf{x})\) is estimated using 4 antithetic replicates, i.e., complementary mask pairs, with masking applied only to non-\textsc{eos} positions.

\paragraph{Training Datasets.}
We train \ourmodel\ on two task families, mathematical reasoning and code generation, with separate task-specific checkpoints for each family. For math, we train on GSM8K~\citep{cobbe2021gsm8k} (7,473 problems) and MATH~\citep{lightman2023math500} (7,500 problems) as two independent training runs. For code, we train on a single mixture of filtered subsets from AceCode-89K~\citep{AceCoder} (11,937 examples) and KodCode-Light-RL-10K~\citep{xu2025kodcode} (2,695 examples). Filtering details are given in App.~\ref{app:dataset_filtering}, and a comparison of training data used by \ourmodel\ and the baselines we evaluate against is provided in ~\cref{tab:training_data_comparison}. Across both domains, rewards are binary: exact answer matching for math and execution-based test passing for code.

\subsection{Evaluation Setup}
\label{sec:eval_setup}

We evaluate on both mathematical reasoning and code generation to test whether improvements transfer across symbolic reasoning and program synthesis. We use the same prompting format for both training and evaluation; representative examples are shown in~App.~\ref{app:prompts}.

\paragraph{Metrics.}
At evaluation time, we generate outputs with maximum lengths of 256 and 512 tokens and a block size of 32. We use the fastdLLM~\citep{wu2025fastdllmtrainingfreeaccelerationdiffusion} parallel decoding sampler with a low-confidence masking strategy: at each denoising step, tokens are selected for unmasking according to model confidence, and decoding proceeds under stochastic nucleus sampling with temperature $T$. We report results under top-$p$=0.91 sampling and compute \textbf{Pass@$k$} from $n$ independently sampled generations per problem using the unbiased estimator of HumanEval~\citep{chen2021humaneval}:
\[
  \mathbf{\bm{Pass}}@k \;=\; 1 - \prod_{i=n-c+1}^{n}\!\left(1 - \frac{k}{i}\right),
\]
where $c \le n$ is the number of correct samples among the $n$ generations. This estimator has lower variance than evaluating exactly $k$ samples directly and approaches the true solve rate as $n$ increases. For math, correctness is determined by exact-match accuracy of the extracted boxed answer; for code, by execution against the provided test suite.

\paragraph{Benchmarks.}
For mathematical reasoning, we evaluate on GSM8K~\citep{cobbe2021gsm8k} (1,319 problems), MATH-500~\citep{lightman2023math500} (500 problems), and the algebra split of Minerva Math~\citep{minerva} (1,187 problems). For code generation, we evaluate on HumanEval~\citep{chen2021humaneval} (164 problems), MBPP~\citep{mbpp} (500 problems), and LiveCodeBench~\citep{jain2025livecodebench} (release\_v5), which contains 880 problems spanning easy, medium, and hard difficulty levels. Together, these benchmarks cover natural-language mathematical reasoning and executable code synthesis, with Minerva Math and LiveCodeBench serving as held-out evaluations beyond the datasets used for post-training.

\paragraph{Baselines.}
Our primary direct baselines are \textsc{JustGRPO}~\citep{ni2026flexibility} and \textsc{ESPO}~\citep{ou2026espo_principled}, two recent post-training methods for diffusion language models with usable public checkpoints. This allows us to evaluate all direct baselines under the same decoding protocol and evaluation pipeline. Other closely related methods, including \textsc{D1}~\citep{zhao2025diffuGRPO} and \textsc{DMPO}~\citep{dmpo}, had not released public checkpoints at the time of our experiments. We therefore include their published numbers in App.~\ref{app:baseline_details} for context. For math, all evaluated methods are trained per-task on GSM8K and MATH separately and reported as matched per-task checkpoints. For code, \ourmodel\ trains on a filtered AceCode-89K and KodCode mixture (14,632 examples total); the public \textsc{JustGRPO} and \textsc{ESPO} coding checkpoints we evaluate were trained on substantially larger coding mixtures than \ourmodel, so coding comparisons are conservative with respect to \ourmodel. Tab.~\ref{tab:training_data_comparison} (App.~\ref{app:dataset_filtering}) summarizes the training data for each method.

\section{Results}
\label{sec:results}
We evaluate \ourmodel\ across four axes: robustness to sampling budget and temperature, in-distribution math and code performance, held-out transfer, and correct-solution diversity, where an LLM-as-a-judge protocol directly tests whether gains reflect broader coverage of distinct solution strategies.

\subsection{Robustness to Sampling Budget and Temperature}
\label{sec:passk_temp}
\ourmodel\ improves over the base diffusion model across all sampling budgets and decoding temperatures we test, including the single-sample setting. We evaluate Pass@\(k\) for \(k\in\{1,2,4,8,16\}\) and \(T\in\{0.3,0.6,0.9\}\), separating two effects that a single Pass@\(k\) number can obscure: whether post-training sacrifices Pass@1 for multi-sample gains, and how performance scales as the sampling budget increases. \Cref{fig:passk_temp} reports average Pass@\(k\) across GSM8K, MATH-500, HumanEval, and MBPP.

For \ourmodel, Pass@\(k\) increases steadily with the sampling budget at all three temperatures (\cref{fig:passk_temp}(a)). Higher-temperature decoding is especially beneficial as \(k\) grows: \(T=0.9\) gives the strongest performance at larger sampling budgets, followed by \(T=0.6\) and \(T=0.3\). This indicates that \ourmodel\ benefits from more exploratory sampling rather than saturating after only a few samples.

The gap to the base model is positive at every \(k\) and \(T\) (\cref{fig:passk_temp}(b)), including \(k=1\). Thus, the multi-sample gains do not come from sacrificing single-sample accuracy. The largest margins occur at \(T=0.9\), where the advantage grows with the sampling budget. The base model itself, however, is strongest at \(T=0.6\). We therefore use \(T=0.6\) for the head-to-head comparison against post-training methods: this gives the base model its best decoding regime and makes the comparison conservative. At this fixed temperature (\cref{fig:passk_temp}(c)), \ourmodel\ leads ESPO, JustGRPO, and the base model from Pass@1 through Pass@16. The shaded regions show variation across evaluation settings, but \ourmodel\ remains the strongest method as the sampling budget increases.
\begin{figure*}[!t]
    \centering
    \includegraphics[width=\textwidth]{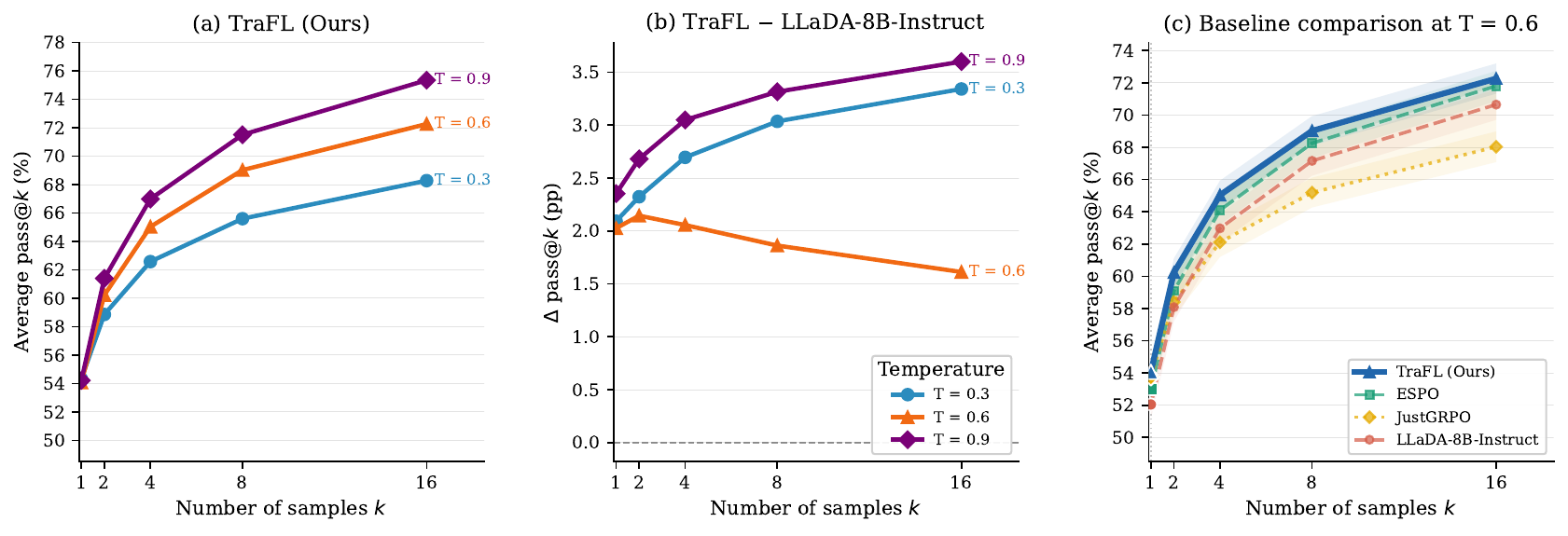}
    \caption{
    \textbf{\ourmodel\ improves over the base model and strong prior post-training methods across sampling budgets and temperatures.}
    (a) Average Pass@\(k\) of \ourmodel\ on GSM8K, MATH-500, HumanEval, and MBPP for \(T\in\{0.3, 0.6, 0.9\}\).
    (b) Pass@\(k\) gap to LLaDA-8B-Instruct under matched decoding. 
    (c) Baseline comparison at \(T=0.6\). \ourmodel\ leads ESPO, JustGRPO, and the base model from Pass@1 through Pass@16. 
    All results use \(n=16\) samples and are averaged over generation lengths 256 and 512.
    }
    \label{fig:passk_temp}
\end{figure*}

\subsection{Pass@5 Results}
\label{sec:main_passk_results}

\ourmodel\ is the only post-training method that improves over the base model in every benchmark-length setting we evaluate. \cref{fig:main_pass5} reports Pass@5 with \(n=16\) samples at \(T=0.6\) for maximum completion lengths 256 and 512, alongside average denoising steps. On math, the largest gain is on MATH-500 at length 256, where \ourmodel\ improves Pass@5 from 50.2 to 54.6, with smaller but positive gains on GSM8K at both lengths. On code, \ourmodel\ improves HumanEval from 53.2 to 54.8 at length 256 and from 53.7 to 55.7 at length 512, and improves MBPP from 59.1 to 62.1 and 59.2 to 62.0 at the two lengths. Averaged over all eight settings, this is a \(+2.0\) point improvement.

The baselines show uneven behavior across benchmarks. ESPO is the strongest prior method overall and improves several settings, especially MATH-500 at length 512, but it drops below the base model on GSM8K-512 and HumanEval-256. JustGRPO improves MBPP substantially but loses performance on GSM8K, HumanEval, and MATH-500 at length 512. \ourmodel\ is the only method with positive improvements in all eight settings, so its average gain reflects broad, consistent improvements rather than being concentrated on a single benchmark.
\begin{figure*}[ht]
    \centering
    \includegraphics[width=\textwidth]{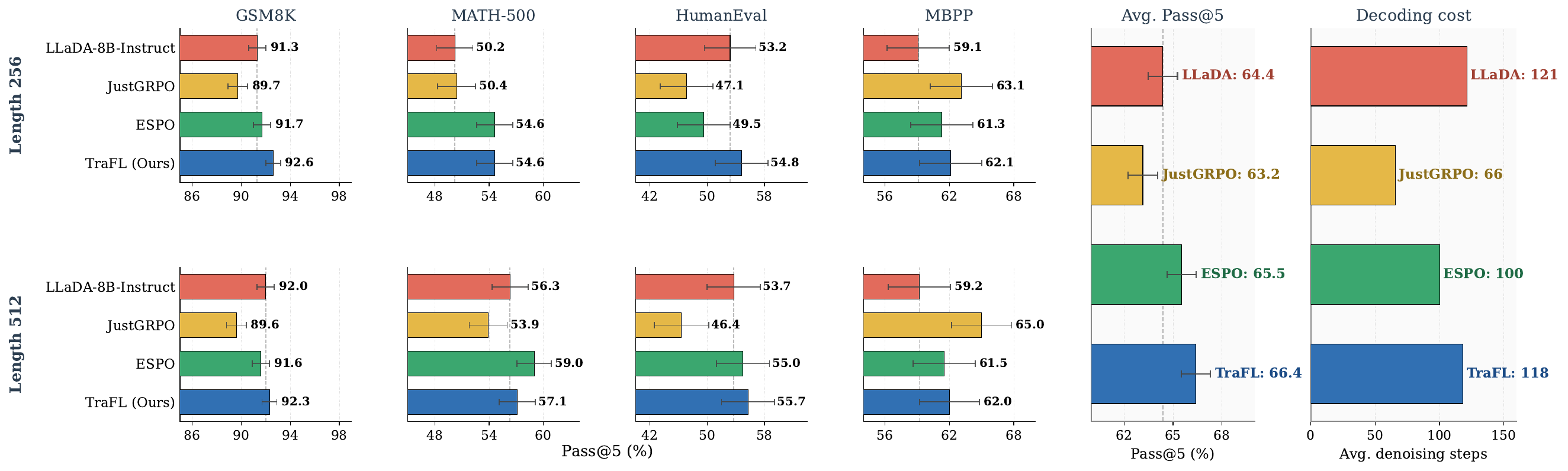}
    \caption{
    \textbf{\ourmodel\ improves Pass@5 across math and coding benchmarks at comparable denoising cost.} 
    Pass@5 on GSM8K, MATH-500, HumanEval, and MBPP at maximum completion lengths 256 and 512 (\(n=16\), \(T=0.6\)). Error bars show standard error. 
    }
    \label{fig:main_pass5}
\end{figure*}
Average denoising steps add a complementary view of the learned policies. LLaDA-8B-Instruct, ESPO, and \ourmodel\ use comparable denoising budgets, while JustGRPO terminates much earlier without consistent Pass@5 gains. Shorter denoising is therefore not by itself a sign of better generation. One possible explanation is that JustGRPO's early termination may reflect a form of premature trajectory locking, where probability mass concentrates early around a narrower set of token choices and leaves fewer opportunities for later denoising steps to explore alternative correct completions. \ourmodel, by contrast, achieves consistent Pass@5 gains while keeping denoising step counts close to the base model and ESPO, suggesting it improves multi-sample coverage through a better allocation of probability mass across successful trajectories.

\subsection{Held-out Benchmark Evaluation}
\label{sec:heldout}

The gains of \ourmodel\ transfer to held-out benchmarks not used during post-training. We evaluate the same checkpoints on the algebra split of Minerva Math~\citep{minerva}, an out-of-distribution math benchmark relative to GSM8K and MATH, and LiveCodeBench~\citep{jain2025livecodebench}, a contamination-free coding benchmark with difficulty-stratified problems. All methods use the same decoding protocol at maximum completion lengths 256 and 512.

\begin{figure*}[ht]
    \centering
    \includegraphics[width=\textwidth]{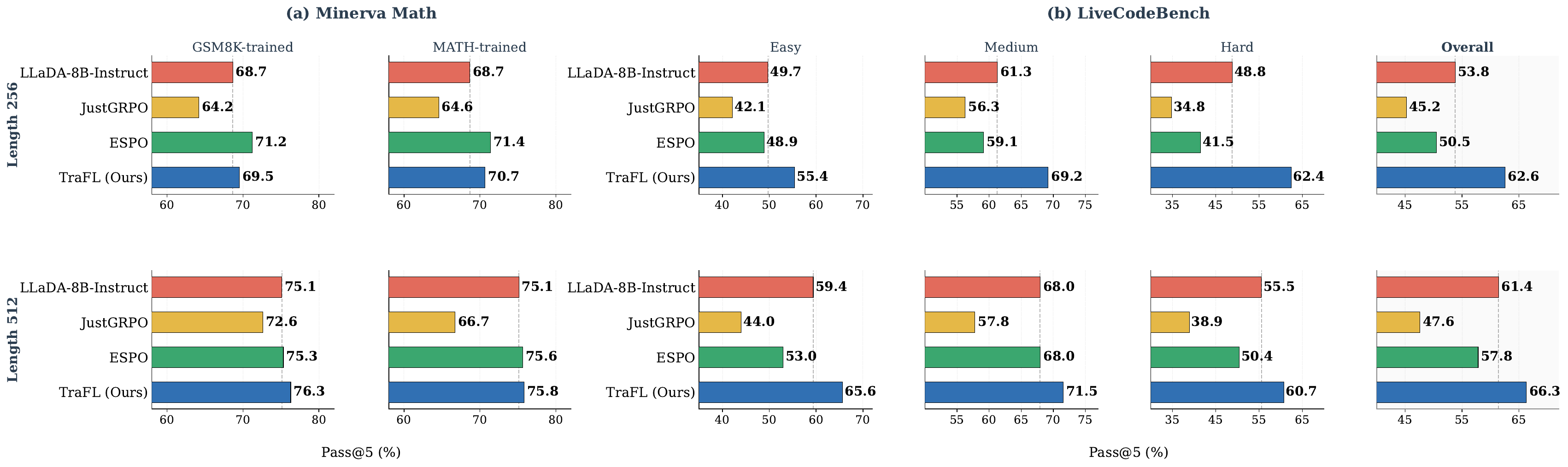}
    \caption{
    \textbf{Gains transfer to held-out math and coding benchmarks.} 
    Pass@5 on Minerva Math (left) using the GSM8K-trained and MATH-trained checkpoints, and on LiveCodeBench (right) by difficulty split, both at maximum completion lengths 256 and 512.
    }
    \label{fig:heldout}
\end{figure*}

On Minerva Math, \ourmodel\ stays above the base model at both lengths and substantially outperforms JustGRPO for both GSM8K-trained and MATH-trained checkpoints. ESPO is stronger at length 256 (71.2--71.4 Pass@5 vs.\ 69.5--70.7 for \ourmodel), but at length 512 \ourmodel\ achieves the best results, reaching 76.3 when trained on GSM8K and 75.8 when trained on MATH.

The pattern is stronger on LiveCodeBench. \ourmodel\ is the best-performing method on every difficulty split and at both generation lengths, improving over the base model from 53.8 to 62.6 overall at length 256 and from 61.4 to 66.3 at length 512. The largest gains appear on medium and hard problems: at length 256, \ourmodel\ improves the base model by 7.9 points on medium and 13.6 points on hard, and outperforms JustGRPO by 12.9 and 27.6 points respectively. ESPO and JustGRPO remain below the base model on the overall LiveCodeBench score at both lengths. The consistent gains observed in the in-distribution analysis therefore generalize beyond the post-training datasets.

\subsection{LLM-as-a-Judge Diversity Evaluation}
\label{sec:llm_as_judge}

Higher Pass@\(k\) does not by itself imply broader mode coverage: a method could improve Pass@\(k\) by sharpening its scoring of a single mode rather than by spreading mass across distinct modes. We therefore directly test whether \ourmodel's gains correspond to coverage of distinct solution strategies, the property that Theorem~\ref{thm:trajectory_diversity} ties to trajectory diversity. We use an LLM-as-a-judge protocol with GPT-5-4 on MATH-500 (IID to our math post-training distribution) and LiveCodeBench (OOD to our coding post-training distribution). For each problem, the judge compares two sets of \(16\) responses \((D_A, D_B)\) generated under matched decoding (max length 256, \(T=0.6\)), with answer-set order randomized to provide the most diverse set (A or B), or tie, \(\mathcal{S}_{LLM}:(D_A,D_B)\rightarrow \{A_{wins}, B_{wins}, Tie\}\). We run pairwise comparisons between \ourmodel\ and each baseline. The full judge prompt is in App.~\ref{app:diversity_judge_prompt}.

The judge evaluates diversity in the underlying solution approach, not surface form. An answer set is preferred only when it contains a broader range of substantively different methods, such as different equation setups, proof strategies, or algorithmic decompositions. The prompt asks the judge to ignore wording, formatting, variable names, verbosity, and minor errors that do not change the core approach. We pre-specify three one-sided sign tests~\citep{conover1999practical} ($p$-value$\leq$0.05) with hypotheses \((\bf H_{alt})\) of the form (a) \(\mathrm{\ourmodel}_{wins}\)$>$ \(\mathrm{Base}_{wins}\), (b) \(\mathrm{ESPO}_{wins}\)$>$\(\mathrm{Base}_{wins}\),and (c) \(\mathrm{\ourmodel}_{wins}\)$>$\(\mathrm{ESPO}_{wins}\) for each benchmark.
\newcommand{\cmark}{\raisebox{0.25ex}{\checkmark}}
\newcommand{\xmark}{\raisebox{0.25ex}{\ding{55}}}
\begin{table*}[!t]
\centering
\caption{
\textbf{LLM-as-a-judge diversity evaluation.}
\ourmodel\ shows the strongest and most consistent diversity gains over correct solutions, where it is preferred over both the base model and ESPO on MATH-500 and LiveCodeBench. ESPO is not significantly more diverse than the base model on either benchmark. ``For''/``Against'' are the fractions of judged problems favoring each side; \(p\)-values are from one-sided exact sign tests over non-tied comparisons.
}
\label{tab:diversity_judge}

\setlength{\tabcolsep}{3.4pt}
\renewcommand{\arraystretch}{1.08}

\begin{adjustbox}{max width=\textwidth}
\begin{tabular}{lccccclccccl}
\toprule
\multirow{2}{*}{\textbf{Benchmark}} 
& \multirow{2}{*}{\(\bf \mathrm{H_{alt}}\)}
& \multicolumn{5}{c}{\textbf{All samples}}
& \multicolumn{5}{c}{\textbf{Correct solutions}} \\
\cmidrule(lr){3-7} \cmidrule(lr){8-12}
& 
& \textbf{For} & \textbf{Against} & \textbf{Tie} & \textbf{\(p\)} & \textbf{Sig.}
& \textbf{For} & \textbf{Against} & \textbf{Tie} & \textbf{\(p\)} & \textbf{Sig.} \\
\midrule

\multirow{3}{*}{MATH-500}
& \ourmodel \(>\) Base
& 21.8 & 15.4 & 62.8 & \(2.3{\times}10^{-2}\) & \cmark
& 12.4 & 5.1 & 82.5 & \(6.0{\times}10^{-3}\) & \cmark \\

& ESPO \(>\) Base
& 18.2 & 17.6 & 64.2 & \(8.8{\times}10^{-1}\) & \textcolor{red}{\xmark}
& 7.5 & 10.0 & 82.6 & \(3.9{\times}10^{-1}\) & \textcolor{red}{\xmark} \\

& \ourmodel \(>\) ESPO
& 21.8 & 14.2 & 64.0 & \(5.7{\times}10^{-3}\) & \cmark
& 14.3 & 5.6 & 80.1 & \(1.0{\times}10^{-3}\) & \cmark \\

\midrule

\multirow{3}{*}{LiveCodeBench}
& \ourmodel \(>\) Base
& 20.9 & 23.0 & 56.1 & \(3.9{\times}10^{-1}\) & \textcolor{red}{\xmark}
& 31.2 & 19.3 & 49.6 & \(4.1{\times}10^{-5}\) & \cmark \\

& ESPO \(>\) Base
& 14.5 & 44.2 & 41.3 & \(1.0{\times}10^{0}\) & \textcolor{red}{\xmark}
& 31.1 & 34.0 & 34.9 & \(7.9{\times}10^{-1}\) & \textcolor{red}{\xmark} \\

& \ourmodel \(>\) ESPO
& 45.9 & 12.2 & 41.9 & \(<10^{-4}\) & \cmark
& 38.3 & 26.4 & 35.3 & \(7.8{\times}10^{-4}\) & \cmark \\

\bottomrule
\end{tabular}
\end{adjustbox}
\end{table*}
\Cref{tab:diversity_judge} separates two notions of diversity. The all-samples setting measures the diversity of the model's sampled behavior overall, including both correct and incorrect responses. Under this view, \ourmodel\ shows a modest but significant gain over the base model on MATH-500, indicating that it explores a broader set of reasoning attempts in the math setting; the signal is not reliable on LiveCodeBench. We do not find support for \(\mathrm{ESPO}_{wins}\)$>$\(\mathrm{Base}_{wins}\). Against ESPO, \ourmodel\ is preferred on both: 21.8\% vs.\ 14.2\% on MATH-500 (\(p=0.0057\)) and 45.9\% vs.\ 12.2\% on LiveCodeBench (\(p<10^{-4}\)).

The correct-solution setting asks the more targeted question, and is the one most directly tied to the trajectory-locking analysis: among responses that solve the problem, does \ourmodel\ cover more distinct valid approaches? We filter each side to its correct samples and keep a comparison only when both sides contain at least one correct solution. Here the signal is stronger and more consistent. Relative to the base model, \ourmodel\ is preferred on both MATH-500 (\(p=0.006\)) and LiveCodeBench (\(p=4.1\times10^{-5}\)). Relative to ESPO, the same holds: 14.3\% vs.\ 5.6\% on MATH-500 (\(p=0.001\)) and 38.3\% vs.\ 26.4\% on LiveCodeBench (\(p=7.8\times10^{-4}\)). The advantage is clearest and most consistent in the correct-solution setting.

Two patterns line up with the trajectory-locking story. ESPO is not significantly more diverse than the base model on either benchmark, despite its Pass@\(k\) gains. This is what we would expect from reward-amplifying optimization without reference anchoring: it improves accuracy by concentrating mass on a narrow support, not by covering more modes. \ourmodel, by contrast, is preferred over ESPO most strongly in the correct-solution setting, where Theorem~\ref{thm:trajectory_diversity} directly applies. The reference-anchored target leaves room for multiple valid reasoning paths to retain significant probability mass, and the diversity gains follow.

\section{Conclusion}
\label{sec:conclusion}

We presented \ourmodel, a trajectory-balance post-training objective for diffusion language models that trains toward a reward-tilted target distribution rather than directly maximizing reward. The key insight is that reward-maximizing objectives are blind to path allocation across denoising trajectories, making sampled updates vulnerable to trajectory locking---a self-reinforcing collapse of path diversity that limits solution coverage. \ourmodel\ avoids this by anchoring to a reference model and learning a prompt-dependent partition function \(Z_\phi(\mathbf{x})\) jointly with the policy. Across mathematical reasoning and code generation benchmarks, \ourmodel\ is the only evaluated method that consistently improves over the base model, outperforming ESPO and JustGRPO on average with gains that grow at larger sampling budgets. The improvements transfer to held-out Minerva Math and LiveCodeBench, and an LLM-as-judge analysis indicates that the gains come with broader coverage of correct solution strategies rather than sharper scoring of a single mode. We hope these results encourage further exploration of distribution-matching objectives for diffusion language model post-training. We discuss limitations and broader impact of our work in App.~\ref{app:limitations_impact}.

\bibliographystyle{plain}
\bibliography{neurips_2026}

\appendix
\newpage\appendix

\textbf{\centering \Large{Overview of Appendix}}

\begin{enumerate}
    \item[A] \textbf{Training Algorithm}
    \item[B] \textbf{Derivation of the Normalized Log-Probability Surrogate}
    \item[C] \textbf{Theory: Path Diversity, Mode Coverage, and Trajectory Balance}
    \item[D] \textbf{Prompt Templates and Sample Rollouts}
    \item [E] \textbf{Dataset Filtering and Training-Data Comparison}
    \item[F] \textbf{Additional Baseline Comparison Details}
    \item[G] \textbf{LLM-as-a-Judge Diversity Prompt}
    \item[H] \textbf{Limitations and Broader Impact}
    \item [I] \textbf{Training Compute}
    \item [J] \textbf{Dataset and Model Licenses}
    
\end{enumerate}

\section{Training Algorithm}
\label{app:algorithm}

Algorithm~\ref{alg:algorithm} summarizes the training procedure for \ourmodel. We use a fully online RL setup: at each step, we sample rollouts from the current policy, compute rewards and surrogate scores on the same batch, and perform exactly one gradient update.

\begin{algorithm}[ht]
\caption{\ourmodel\ training algorithm}
\label{alg:algorithm}
\KwIn{initial policy \(p_{\theta_{\mathrm{init}}}\); frozen reference \(p_{\mathrm{ref}}\); reward \(r\); prompts \(\mathcal{D}\); hparams \(\beta,G,K,M\)}
\KwOut{trained policy \(p_{\theta}\) and normalization predictor \(Z_{\phi}\)}

\(p_{\theta}\gets p_{\theta_{\mathrm{init}}}\), initialize \(Z_{\phi}\)\;

\For{step \(=1\) \textbf{to} \(M\)}{
    Sample prompt batch \(\mathcal{D}_b\sim\mathcal{D}\)\;
    
    \ForEach{\(\mathbf{x}\in\mathcal{D}_b\)}{
        Sample \(\{\mathbf{y}^{(g)}\}_{g=1}^{G}\sim p_{\theta}(\cdot\mid\mathbf{x})\)\;
        
        Compute \(r^{(g)}\gets r(\mathbf{x},\mathbf{y}^{(g)})\) and
        \(\widetilde r^{(g)}\gets r^{(g)}-\frac{1}{G}\sum_{g'=1}^{G}r^{(g')}\)\;
        
        Estimate \({\log \hat p}_{\theta}(\mathbf{y}^{(g)}\mid\mathbf{x})\) and
        \({\log \hat p}_{\mathrm{ref}}(\mathbf{y}^{(g)}\mid\mathbf{x})\)
        with \cref{eq:seq_surrogate} using \(K\) mask samples\;
        
        Compute
        \[
        \delta^{(g)}
        =
        {\log \hat p}_{\theta}(\mathbf{y}^{(g)}\mid\mathbf{x})
        -
        {\log \hat p}_{\mathrm{ref}}(\mathbf{y}^{(g)}\mid\mathbf{x})
        -
        \beta\,\widetilde r^{(g)}
        +
        \log Z_{\phi}(\mathbf{x}) .
        \]
    }
    Update \((\theta,\phi)\) by minimizing
    \[
    \mathcal{L}_{\mathrm{TraFL}}
    =
    \frac{1}{|\mathcal{D}_b|G}
    \sum_{\mathbf{x}\in\mathcal{D}_b}
    \sum_{g=1}^{G}
    \left(\delta^{(g)}\right)^2 .
    \]
}
\end{algorithm}

\section{Derivation of the Normalized Log-Probability Surrogate}
\label{app:surrogate_elbo}

We derive \cref{eq:seq_surrogate} from a masked-reconstruction lower bound on a one-step reconstruction marginal, and explain the role of the \(1/l\) factor.

\paragraph{Setup.}
Let \(\mathbf{y}=(y_1,\dots,y_L)\) be a completion of length \(L\). For each mask level \(l\in\{1,\dots,L\}\), \(q_l(\mathbf{y}_l\mid \mathbf{y})\) selects \(l\) completion-token positions uniformly without replacement and replaces them with \(\mathtt{[mask]}\), so
\[
q_l(\mathbf{y}_l \mid \mathbf{y})=\binom{L}{l}^{-1}
\]
on consistent corruptions and zero elsewhere. The model predicts all masked positions in one forward pass, with
\[
p_\theta(\mathbf{y}\mid \mathbf{y}_l,\mathbf{x})
=
\prod_{i:\,y_i^l=\mathtt{[mask]}}
p_\theta(y_i\mid \mathbf{y}_l,\mathbf{x}).
\]

\paragraph{One-step reconstruction lower bound.}
Define the one-step masked reconstruction marginal at level \(l\) as
\[
p_\theta^{(l)}(\mathbf{y}\mid \mathbf{x})
=
\mathbb{E}_{\mathbf{y}_l\sim q_l(\cdot\mid \mathbf{y})}
\bigl[p_\theta(\mathbf{y}\mid \mathbf{y}_l,\mathbf{x})\bigr].
\]
Jensen's inequality gives
\begin{equation}
\label{eq:oneshot_jensen}
\log p_\theta^{(l)}(\mathbf{y}\mid \mathbf{x})
\ge
\mathbb{E}_{\mathbf{y}_l\sim q_l(\cdot\mid \mathbf{y})}
\!\left[
\sum_{i=1}^{L}
\mathbf{1}[y_i^l=\mathtt{[mask]}]
\log p_\theta(y_i\mid \mathbf{y}_l,\mathbf{x})
\right].
\end{equation}
The indicator picks out exactly \(l\) terms, so this lower bound is the expected sum of \(l\) masked-token log-probabilities. Averaging over \(l\sim\mathrm{Uniform}(\{1,\dots,L\})\) gives
\begin{equation}
\label{eq:avg_over_l_bound}
\mathbb{E}_l\!\left[\log p_\theta^{(l)}(\mathbf{y}\mid \mathbf{x})\right]
\ge
\mathbb{E}_l\,\mathbb{E}_{\mathbf{y}_l}\!\left[
\sum_{i=1}^{L}
\mathbf{1}[y_i^l=\mathtt{[mask]}]
\log p_\theta(y_i\mid \mathbf{y}_l,\mathbf{x})
\right].
\end{equation}

\paragraph{Normalization by \(1/l\).}
\Cref{eq:seq_surrogate} differs from the right-hand side of \cref{eq:avg_over_l_bound} by a factor of \(1/l\) inside the inner bracket. Since the inner sum has \(l\) nonzero terms, \(\frac{1}{l}\sum_{i:\,y_i^l=\mathtt{[mask]}}\log p_\theta(y_i\mid \mathbf{y}_l,\mathbf{x})\) is the arithmetic mean of those \(l\) log-probabilities. \(\log\hat p_\theta(\mathbf{y}\mid\mathbf{x})\) is the expectation of this mean under the random corruption \(\mathbf{y}_l\) and mask level \(l\).

The unnormalized sum scales linearly with the random number of masked positions: at higher mask levels the sum has more (negative) terms, so its magnitude varies with \(l\) even when per-token reconstruction quality does not. Substituting it into \cref{eq:tb_residual} would make the residual size depend on this count. The \(1/l\) factor cancels this dependence, leaving a per-position score. Note that \(\log\hat p_\theta\) is not a length-normalization of any sequence log-likelihood: the denominator is the random count \(l\), not the completion length \(L\), and the inner quantity is a mean over \(l\) masked positions rather than a divided sequence log-probability.

\paragraph{Relation to length-bias diagnostics in autoregressive RL.}
A similar concern about the role of normalization appears in autoregressive RL post-training. Dr.\ GRPO~\citep{liu2025drgrpo} shows that GRPO's per-completion loss normalization by \(|o_i|\) introduces a length bias because \(|o_i|\) varies across the rollout group, and replaces it with a constant. The \(1/l\) factor here plays an analogous role: when a loss aggregates a variable-size set of token-level quantities, the magnitude of the aggregate inherits a count-dependence that is unrelated to per-token quality. The principle is the same in both cases --- normalize so the score is independent of the set size --- though the count being normalized differs (completion length there, mask count here).
\section{Theory: Path Diversity, Mode Coverage, and Trajectory Balance}
\label{app:theory}

This appendix provides full proofs of the main-paper results and extended discussion. The argument in this section proceeds in three steps: (1)~terminal reward is indifferent to path allocation (\cref{prop:path_indifference}); (2)~sampled policy-gradient updates exploit this indifference through trajectory locking; (3)~trajectory locking directly limits terminal mode coverage (\cref{thm:trajectory_diversity,cor:locking}). The subsections below prove each step in turn and then explain how the trajectory-balance perspective motivates \ourmodel. 

\subsection{Full proofs of the main-paper results}
\label{app:theory_mainproofs}

\begin{repeatedtheorembox}
\textbf{Proposition~\ref{prop:path_indifference} (Path-indifference of terminal reward).}
Fix a prompt \(\mathbf{x}\) and a final completion \(\mathbf{y}\). Let \(\mathcal{T}(\mathbf{y})\) denote the set of denoising trajectories that terminate at \(\mathbf{y}\). Any redistribution of probability mass among trajectories in \(\mathcal{T}(\mathbf{y})\) that preserves \(p_\theta(\mathbf{y}\mid\mathbf{x})\) leaves a terminal-reward objective unchanged.
\end{repeatedtheorembox}

\begin{proof}
The terminal-reward objective decomposes as
\[
J(\mathbf{x})=\sum_{\mathbf{y}} p_\theta(\mathbf{y}\mid \mathbf{x})\, r(\mathbf{x},\mathbf{y}),
\qquad
p_\theta(\mathbf{y}\mid \mathbf{x})=\sum_{\tau\in\mathcal{T}(\mathbf{y})} p_\theta(\tau\mid \mathbf{x}).
\]
The contribution of \(\mathbf{y}\) to \(J\) depends only on the scalar \(p_\theta(\mathbf{y}\mid\mathbf{x})\). Therefore, any redistribution of mass among trajectories in \(\mathcal{T}(\mathbf{y})\) that preserves this sum leaves the objective unchanged.
\end{proof}

\begin{repeatedtheorembox}
\textbf{Theorem~\ref{thm:trajectory_diversity} (Trajectory diversity is necessary for mode coverage).}
Let \(\mathcal{T}\) be a random denoising trajectory taking values in a countable set \(\Omega_{\mathcal{T}}\), and let \(M=g(\mathcal{T})\) denote its terminal solution mode, where \(g:\Omega_{\mathcal{T}}\to\Omega_M\) is a deterministic mapping. Then
\[
|\mathrm{supp}(M)| \le |\mathrm{supp}(\mathcal{T})|
\qquad\text{and}\qquad
H(M)\le H(\mathcal{T}),
\]
where \(\mathrm{supp}(\cdot)\) denotes the support of a distribution and \(H(\cdot)\) the Shannon entropy. In particular, terminal mode coverage cannot exceed trajectory diversity.
\end{repeatedtheorembox}

\begin{proof}
Because \(M=g(\mathcal{T})\) is a deterministic function of \(\mathcal{T}\), every mode \(m\in\mathrm{supp}(M)\) must be the image of at least one trajectory \(\tau\in\mathrm{supp}(\mathcal{T})\). Hence \(\mathrm{supp}(M)=g(\mathrm{supp}(\mathcal{T}))\), which gives \(|\mathrm{supp}(M)| \le |\mathrm{supp}(\mathcal{T})|\).

For entropy, since \(M\) is a deterministic function of \(\mathcal{T}\), we have \(H(M\mid \mathcal{T})=0\). By the chain rule,
\[
H(\mathcal{T},M)=H(\mathcal{T})+H(M\mid \mathcal{T})=H(\mathcal{T}).
\]
On the other hand, \(H(\mathcal{T},M)=H(M)+H(\mathcal{T}\mid M)\ge H(M)\). Combining gives \(H(M)\le H(\mathcal{T})\).
\end{proof}

\begin{repeatedtheorembox}
\textbf{Corollary~\ref{cor:locking} (Trajectory locking upper-bounds mode coverage).}
If optimization reduces the trajectory distribution to a subset \(\mathcal{S}\subseteq\Omega_{\mathcal{T}}\), then the set of terminal modes that can be covered is at most \(g(\mathcal{S})\subseteq\Omega_M\). In particular, collapse to a single trajectory implies collapse to a single terminal mode.
\end{repeatedtheorembox}

\begin{proof}
Immediate from \cref{thm:trajectory_diversity}, since \(M=g(\mathcal{T})\).
\end{proof}

\subsection{Sampled reward optimization and trajectory locking}
\label{app:trajectory_locking}

\cref{prop:path_indifference} is an objective-level statement: terminal reward is flat over the entire subspace of trajectory distributions that assign the same mass to each completion. In practice, however, optimization proceeds through sampled updates, and this flatness becomes a liability.

For terminal reward, the policy-gradient update takes the form
\begin{equation}
\label{eq:policy_gradient_terminal}
\nabla J_{\mathrm{RL}}
=
\mathbb{E}_{\tau\sim p_\theta(\cdot\mid \mathbf{x})}
\left[
r(\mathbf{x},\mathbf{y}(\tau))\,\nabla \log p_\theta(\tau\mid \mathbf{x})
\right].
\end{equation}
Only sampled trajectories receive direct gradient signal. Because \(r(\mathbf{x},\mathbf{y})\) is constant across all trajectories reaching the same completion \(\mathbf{y}\), the gradient provides no force to spread mass across different paths to the same answer. Instead, if two trajectories both reach a rewarding completion but one is sampled slightly more often early in training, it receives more positive updates, becomes more likely, and is sampled even more in subsequent steps. This self-reinforcing feedback loop---\emph{trajectory locking}---concentrates probability mass onto a narrow subset of paths despite the underlying objective being indifferent among them.

Trajectory locking is an optimization phenomenon, not a formal consequence of \cref{prop:path_indifference}. The proposition establishes the flat landscape; locking describes how stochastic gradient ascent can move along that landscape toward a vertex or low-dimensional face. We therefore use it as a mechanism explaining why sampled reward maximization may reduce coverage in practice, not as a theorem that exact expected policy-gradient updates must always collapse.

\subsection{A generalized simplex view of trajectory concentration}
\label{app:replicator}

The discussion above concerns strict terminal reward, under which all trajectories leading to the same terminal completion receive the same reward. The following analysis considers a more general setting in which optimization induces effective trajectory-level scores, for example through sampling asymmetries, surrogate objectives, or path-dependent credit assignment. It should therefore be viewed as an auxiliary lens on trajectory concentration rather than a literal description of pure terminal reward.

Fix a completion \(\mathbf{y}\), and let
\[
\mathcal{T}(\mathbf{y})=\{\tau_1,\dots,\tau_K\}.
\]
Define the conditional trajectory distribution
\[
q_i := p_\theta(\tau_i\mid \mathbf{y},\mathbf{x}),
\qquad
q=(q_1,\dots,q_K)\in \Delta^K,
\]
and suppose each trajectory \(\tau_i\) is assigned an effective scalar score \(r_i\). Consider the linear objective
\begin{equation}
\label{eq:linear_simplex_objective}
J(q)=\sum_{i=1}^K q_i r_i.
\end{equation}
Because \cref{eq:linear_simplex_objective} is linear in \(q\), its maximizers lie on extreme points of the simplex when one score dominates, and on a face of the simplex when several scores tie. Thus the objective itself does not prefer dispersed trajectory distributions.

Now parameterize \(q\) by logits \(\phi\) through
\[
q_i=\frac{e^{\phi_i}}{\sum_j e^{\phi_j}}.
\]
The Euclidean gradient of \(J\) with respect to \(\phi\) is
\begin{equation}
\label{eq:logit_gradient}
\frac{\partial J}{\partial \phi_i}=q_i(r_i-\bar r),
\qquad
\bar r=\sum_j q_j r_j.
\end{equation}
\Cref{eq:logit_gradient} already shows that higher-than-average reward increases the corresponding logit. To obtain the standard replicator dynamics in probability space, however, one must specify a particular continuous-time flow in logit space, and this flow is \emph{not} the ordinary Euclidean gradient flow on the logits. Instead, consider \emph{additive payoff dynamics},
\begin{equation}
\label{eq:additive_payoff}
\dot{\phi}_i = r_i.
\end{equation}
This choice corresponds to mirror-descent-style updates on the simplex, and equivalently to a natural-gradient geometry in probability space \citep{amari1998natural,beck2003mirror}. Under this assumed logit flow, the induced dynamics of \(q_i\) are obtained by the chain rule:
\[
\dot{q}_i
=
\sum_j \frac{\partial q_i}{\partial \phi_j}\,\dot{\phi}_j,
\qquad
\frac{\partial q_i}{\partial \phi_j}
=
q_i(\delta_{ij}-q_j).
\]
Substituting \(\dot{\phi}_j = r_j\) gives
\[
\dot{q}_i
=
\sum_j q_i(\delta_{ij}-q_j)\,r_j
=
q_i\bigl(r_i-\bar r\bigr),
\]
which is the replicator equation:
\begin{equation}
\label{eq:replicator}
\dot q_i = q_i(r_i-\bar r).
\end{equation}
More generally, under \(\dot\phi_i = \eta r_i\), the induced flow
\(
\dot q_i = \eta\, q_i(r_i-\bar r)
\)
is proportional to the replicator vector field and differs only by a rescaling of time. The solution to \cref{eq:replicator} is
\begin{equation}
\label{eq:replicator_solution}
q_i(t)=\frac{q_i(0)\,e^{r_i t}}{\sum_{j=1}^K q_j(0)\,e^{r_j t}},
\end{equation}
so relative mass evolves as
\[
\frac{q_i(t)}{q_j(t)}
=
\frac{q_i(0)}{q_j(0)}\,e^{(r_i-r_j)t}.
\]
When \(r_i>r_j\), this ratio grows exponentially, driving the conditional trajectory distribution toward a simplex vertex or a low-dimensional face. This makes precise the sense in which reward-style optimization can be mode-seeking when effective path-level growth rates differ.

\paragraph{Remark on standard Euclidean gradient ascent.}
For completeness, we note that ordinary Euclidean gradient ascent on the logits, \(\dot{\phi}_i \propto \partial J/\partial \phi_i = q_i(r_i-\bar r)\), instead induces the cubic dynamic
\[
\dot{q}_i \propto q_i\!\left[q_i(r_i-\bar r) - \sum_j q_j^2(r_j-\bar r)\right],
\]
which does not admit the closed-form exponential reweighting in \cref{eq:replicator_solution}. We work with the replicator form throughout because it admits a clean exponential solution and makes the mode-seeking geometry of reward-style optimization transparent. The qualitative concentration conclusion---that higher-scoring trajectories acquire exponentially more mass relative to lower-scoring ones---is robust to this choice: both flows drive \(q\) toward simplex vertices or low-dimensional faces when scores differ. The replicator equation should therefore be viewed as an idealized simplification of the actual logit dynamics, chosen for analytical tractability rather than as the literal flow induced by standard policy-gradient updates.

\subsection{An idealized trajectory-balance interpretation}
\label{app:idealized_tb}

This subsection describes an idealized exact trajectory-balance view that motivates our design. \ourmodel\ is inspired by this principle, but in practice uses a diffusion-compatible sequence-level surrogate and a learned prompt-dependent normalization term rather than exact trajectory probabilities.

Let \(P_F(\tau)\) denote a forward trajectory distribution and let \(P_B(\tau\mid \mathbf{y},\mathbf{x})\) denote a backward distribution over trajectories conditioned on a terminal completion \(\mathbf{y}\). In an exact trajectory-balance formulation, one enforces
\begin{equation}
\label{eq:idealized_tb}
\log Z(\mathbf{x}) + \log P_F(\tau\mid \mathbf{x})
=
\beta r(\mathbf{x},\mathbf{y}) + \log P_B(\tau\mid \mathbf{y},\mathbf{x}),
\qquad \forall \tau\in\mathcal{T}(\mathbf{y}).
\end{equation}
Equivalently,
\begin{equation}
\label{eq:idealized_tb_factorization}
P_F(\tau\mid \mathbf{x})
=
\frac{1}{Z(\mathbf{x})}
\exp\!\bigl(\beta r(\mathbf{x},\mathbf{y})\bigr)\,
P_B(\tau\mid \mathbf{y},\mathbf{x}).
\end{equation}
Summing over \(\tau\in\mathcal{T}(\mathbf{y})\) yields
\begin{equation}
\label{eq:terminal_mass_reward_tilt}
P_F(\mathbf{y}\mid \mathbf{x})
=
\frac{1}{Z(\mathbf{x})}
\exp\!\bigl(\beta r(\mathbf{x},\mathbf{y})\bigr).
\end{equation}
Therefore, reward determines only the \emph{total} mass assigned to terminal completion \(\mathbf{y}\). Conditioned on \(\mathbf{y}\), we obtain
\begin{equation}
\label{eq:conditional_path_allocation}
P_F(\tau\mid \mathbf{y},\mathbf{x})=P_B(\tau\mid \mathbf{y},\mathbf{x}).
\end{equation}
This idealized factorization separates terminal weighting from within-terminal path allocation. Unlike direct reward amplification, it does not repeatedly favor one trajectory over another merely because they terminate at the same rewarding output.

We stress again that \cref{eq:idealized_tb,eq:idealized_tb_factorization,eq:terminal_mass_reward_tilt,eq:conditional_path_allocation} describe an exact idealized trajectory-balance picture. They are intended as conceptual motivation for \ourmodel, not as exact identities satisfied by the practical surrogate objective in \cref{sec:method}.

\subsection{Relation to the practical \ourmodel\ objective}
\label{app:theory_to_method}

The sections above build the theoretical case, but \ourmodel's practical objective differs from the idealized picture in two important ways.

\textbf{Reference model instead of explicit backward model.} The idealized TB factorization in \cref{app:idealized_tb} uses a backward distribution \(P_B(\tau\mid\mathbf{y},\mathbf{x})\) to govern within-completion path allocation. In \ourmodel, this role is played by the frozen reference model: the reward-tilted target
\[
p^*(\tau\mid \mathbf{x})
\propto
p_{\mathrm{ref}}(\tau\mid \mathbf{x})
\exp\!\bigl(\beta r(\mathbf{x},\mathbf{y})\bigr)
\]
factors, after marginalization over \(\tau\in\mathcal{T}(\mathbf{y})\), into
\begin{enumerate}[label=(\roman*)]
    \item a completion-level marginal proportional to \(\exp\!\bigl(\beta r(\mathbf{x},\mathbf{y})\bigr)\,p_{\mathrm{ref}}(\mathbf{y}\mid\mathbf{x})\), which determines how much total mass each completion receives, jointly through the reward tilt and the reference model's marginal over completions; and
    \item a within-completion conditional path distribution \(p_{\mathrm{ref}}(\tau\mid \mathbf{y},\mathbf{x})\), which governs how that mass is spread across the denoising trajectories that reach the same completion.
\end{enumerate}
This is precisely the separation that policy-gradient reward maximization lacks: reward acts only on (i), while (ii) is left undefined by the objective and is determined implicitly by sampling, enabling trajectory locking.

\textbf{Completion-level surrogate instead of exact trajectory scores.} Diffusion language models do not expose exact trajectory log-probabilities, so the practical residual in \cref{eq:tb_residual} uses the masked-token surrogate \({\log \hat p}_\theta(\mathbf{y}\mid\mathbf{x})\) from \cref{eq:seq_surrogate} together with a learned prompt-dependent predictor for \(\log Z(\mathbf{x})\). As a result, the implemented loss compares current and reference models at the level of sampled completions, not at the level of latent denoising paths. The practical objective is therefore inspired by the idealized trajectory-balance equation and inherits its intuition, but it does not literally enforce equality of within-completion path distributions.

\textbf{Centered rewards as variance reduction.} In the implementation, we replace \(r(\mathbf{x},\mathbf{y})\) by rewards centered within each prompt group. Because the centering term depends on the sampled group rather than only on \(\mathbf{x}\), it cannot in general be absorbed exactly into the deterministic normalization term \(Z(\mathbf{x})\). This centering should therefore be interpreted as a stochastic variance-reduction heuristic for optimization, not as an exact rewriting of the idealized target distribution.

\subsection{Why DMPO's batch-softmax \texorpdfstring{$\hat{Z}$}{Z-hat} still induces trajectory locking}
\label{app:dmpo_replicator}

DMPO~\citep{dmpo} targets the same reward-tilted distribution \(p^*\) as \ourmodel\ but estimates \(Z(\mathbf{x})\) without a learned network, instead normalizing importance weights by a softmax over the current training buffer. Concretely, for each rollout \(o^{(n)}\) it defines
\[
w^{(n)} = \frac{\exp(\ell^{(n)})}{\sum_{k} \exp(\ell^{(k)})},
\qquad
\ell^{(n)} = \frac{r^{(n)}}{\alpha} + \log\frac{p_{\mathrm{ref}}(o^{(n)}\mid\mathbf{x})}{p_v(o^{(n)}\mid\mathbf{x})},
\]
and minimizes \(\sum_n w^{(n)} \mathcal{L}_{\mathrm{DCE}}(o^{(n)})\). This batch-softmax step is an empirical estimator of \(\hat{Z}(\mathbf{x})\).

Treating the weights \(w^{(n)}\) as stop-gradient targets (as in DMPO's implementation), the WDCE objective drives the buffer logits \(\phi^{(n)} = \log p_\theta(o^{(n)}\mid\mathbf{x})\) according to additive payoff dynamics
\[
\dot{\phi}^{(n)} \propto w^{(n)}.
\]
This is precisely the additive payoff regime analyzed in \cref{app:replicator}: under additive logit dynamics, the chain rule through the softmax Jacobian yields the replicator equation on the conditional buffer distribution,
\[
\dot{q}^{(n)} = q^{(n)}\bigl(w^{(n)} - \bar{w}\bigr),
\qquad
q^{(n)} = \frac{p_\theta(o^{(n)}\mid\mathbf{x})}{\sum_k p_\theta(o^{(k)}\mid\mathbf{x})},
\]
with solution
\[
q^{(n)}(t) = \frac{q^{(n)}(0)\,e^{w^{(n)} t}}{\sum_{k} q^{(k)}(0)\,e^{w^{(k)} t}}.
\]
The highest-weight completion therefore acquires exponentially more mass relative to all others in the buffer, and any completion not present in the buffer receives zero gradient signal. As the policy concentrates, future buffers are sampled from an increasingly narrow distribution, so unseen modes remain unseen. Unlike \ourmodel's learned \(Z_\phi\), which receives gradient signal at every step regardless of which completions are sampled, the batch-softmax \(\hat{Z}\) degrades to a point estimate over the locked support, removing the correction signal that would otherwise counteract collapse.

\section{Prompt Templates and Sample Rollouts}
\label{app:prompts}

Below, we present two representative sampled trajectories, one from the coding domain and one from mathematical reasoning. Each trajectory consists of a task prompt and a completion sampled from the current policy. We use the same prompting format for both training and evaluation: all task instructions and response-format constraints are specified directly in the user prompt, with no separate system prompt.

\begin{tcolorbox}[code_prompt_rollout]
\tiny
\textbf{\color{systemgray}Rollout Example: User Prompt and Sample Model Output}\\[0.4em]

{\color{tealprompt}\textbf{[User Prompt]}}\\
{\color{tealprompt}You are a coding expert. You will be given a coding problem to solve.}\\
{\color{tealprompt}Solve it step by step. Ensure you wrap the answer in \texttt{```python```}.}\\[0.5em]

{\color{tealprompt}\textbf{Response Format}}\\
{\color{tealprompt}\texttt{<reasoning>}} Your reasoning here {\color{tealprompt}\texttt{</reasoning>}}\\
{\color{tealprompt}\texttt{<answer>}} \texttt{```python} ... \texttt{```} {\color{tealprompt}\texttt{</answer>}}\\[0.75em]

{\color{tealprompt}\textbf{[Task Prompt]}}\\
{\color{tealprompt}Given a list of subject IDs and a list of trial types, create a function that generates a dataset as a dictionary. Each subject ID is a key, and its value is another dictionary mapping trial types to strings formatted as \texttt{Subject \{subject\_id\}, Trial \{trial\_type\}}.}\\[0.75em]

{\color{purpleprompt}\textbf{[Model Output]}}\\
{\color{reasoningcolor}\texttt{<reasoning>}}\\
To solve this problem: (1) iterate over subject IDs; (2) create a nested dictionary for each subject; (3) map each trial type to the required formatted string; and (4) return the completed dataset.\\
{\color{reasoningcolor}\texttt{</reasoning>}}\\[0.4em]

{\color{reasoningcolor}\texttt{<answer>}}
\begin{lstlisting}[style=mypython]
def generate_dataset(subject_ids, trial_types):
    dataset = {}
    for subject_id in subject_ids:
        dataset[subject_id] = {}
        for trial_type in trial_types:
            dataset[subject_id][trial_type] = (
                f"Subject {subject_id}, Trial {trial_type}"
            )
    return dataset
\end{lstlisting}
{\color{reasoningcolor}\texttt{</answer>}}
\end{tcolorbox}

\begin{tcolorbox}[code_prompt_rollout]
\tiny
\textbf{\color{systemgray}Rollout Example: User Prompt and Sample Model Output}\\[0.4em]

{\color{tealprompt}\textbf{[User Prompt]}}\\
{\color{tealprompt}You are a math expert. You will be given a question to solve.}\\
{\color{tealprompt}Solve it step by step. Wrap the final answer in \texttt{\textbackslash boxed\{\}} inside the <answer> tags..}\\[0.5em]

{\color{tealprompt}\textbf{Response Format}}\\
{\color{tealprompt}\texttt{<reasoning>}} Your reasoning here {\color{tealprompt}\texttt{</reasoning>}}\\
{\color{tealprompt}\texttt{<answer>}} \texttt{\textbackslash boxed\{...\}} {\color{tealprompt}\texttt{</answer>}}\\[0.75em]

{\color{tealprompt}\textbf{[Task Prompt]}}\\
{\color{tealprompt}Mark makes custom dog beds. A bed for a Rottweiler takes 8 pounds of stuffing, a bed for a chihuahua takes 2 pounds of stuffing, and a bed for a collie takes the average amount of stuffing between the first two kinds of beds. How many pounds of stuffing does Mark need to make 4 chihuahua beds and 3 collie beds?}\\[0.75em]

{\color{purpleprompt}\textbf{[Model Output]}}\\
{\color{reasoningcolor}\texttt{<reasoning>}}\\
First, a collie bed uses the average amount of stuffing between a Rottweiler bed and a chihuahua bed, so it needs
\[
\frac{8 + 2}{2} = 5
\]
pounds.

Next, 4 chihuahua beds require
\[
4 \times 2 = 8
\]
pounds, and 3 collie beds require
\[
3 \times 5 = 15
\]
pounds.

Therefore, the total amount of stuffing needed is
\[
8 + 15 = 23.
\]
{\color{reasoningcolor}\texttt{</reasoning>}}\\[0.4em]

{\color{reasoningcolor}\texttt{<answer>}}\\
\[
\boxed{23}
\]
{\color{reasoningcolor}\texttt{</answer>}}
\end{tcolorbox}

\section{Dataset Filtering and Training-Data Comparison}
\label{app:dataset_filtering}

This appendix gives the filtering procedures used to construct \ourmodel's coding training set and summarizes how the training data of \ourmodel\ relates to the data used by the public baseline checkpoints we compare against.

\subsection{Filtering Procedures}
\label{app:filtering_procedures}

\paragraph{AceCode-89K.}
We use existing CodeLlama-Instruct~\citep{codellamaopenfoundation} rollout metadata distributed with AceCode-89K to filter problems by reward uncertainty under the proxy model. We retain problems with substantial reward uncertainty under that proxy, removing both trivially easy problems (uniformly passed) and uniformly failed problems for which binary feedback would not be informative. This yields 11,937 examples.

\paragraph{KodCode-Light-RL-10K.}
For KodCode, we compute per-example pass-rate statistics from GPT-4o and DeepSeek-R1~\citep{deepseekr1} rollouts. We identify the weaker proxy on average, retain examples with high pass-sequence variance under that weaker proxy, and then keep the lower-pass-rate portion of this diverse subset. This focuses training on harder coding problems while still admitting informative binary feedback. The resulting subset contains 2,695 examples.

\paragraph{Combined coding training set.}
\ourmodel's coding model is trained on the concatenation of these two filtered subsets (14,632 examples total). The filtering targets problems with informative binary feedback rather than removing examples on stylistic or domain grounds.

\subsection{Training-Data Comparison Across Methods}
\label{app:training_data_comparison}

The public checkpoints of our direct baselines (\textsc{JustGRPO} and \textsc{ESPO}) were trained on different data mixtures than \ourmodel. Tab.~\ref{tab:training_data_comparison} summarizes the training data used by each method we compare against, so that algorithmic and data differences are not conflated when reading the empirical results.

\begin{table}[h]
\centering
\small
\caption{Training-data summary across the methods we compare against. All methods train separate task-specific checkpoints for GSM8K and for MATH (reported separately throughout the paper). For coding, the public \textsc{JustGRPO} and \textsc{ESPO} checkpoints we evaluate were trained on substantially larger coding mixtures than \ourmodel.}
\label{tab:training_data_comparison}
\setlength{\tabcolsep}{4pt}
\begin{tabular}{llll}
\toprule
Method & GSM8K-trained & MATH-trained & Code training data \\
\midrule
\textsc{JustGRPO}~\citep{ni2026flexibility} & GSM8K (7,473) & MATH (7,500) & AceCode-87K (21,000 hard samples) \\
\textsc{ESPO}~\citep{ou2026espo_principled} & GSM8K (7,473) & MATH (7,500) & AceCode-87K (21,000 hard samples) \\
\multirow{2}{*}{\ourmodel\ (Ours)}
  & \multirow{2}{*}{GSM8K (7,473)} & \multirow{2}{*}{MATH (7,500)}
  & AceCode-89K filtered (11,937) \\
  & & & + KodCode filtered (2,695) \\
\bottomrule
\end{tabular}
\end{table}

Two aspects of this comparison are worth noting. First, all methods follow the same per-task convention for math, training one checkpoint on GSM8K and a separate checkpoint on MATH, so the math comparisons are matched in training-data scale and source. Second, the coding training set used by \ourmodel\ (14,632 examples) is smaller than the coding mixtures used to produce the public \textsc{JustGRPO} and \textsc{ESPO} coding checkpoints, so coding comparisons are conservative with respect to \ourmodel: any algorithmic advantage we report on coding holds despite \ourmodel\ being trained on less data than the baselines we compare against. None of these training-data differences affect the held-out evaluations on Minerva Math and LiveCodeBench, where the conclusions reported in the main text are drawn under matched decoding settings on benchmarks that no method was trained on directly.

\section{Additional Baseline Comparison Details}
\label{app:baseline_details}

In our main experiments, we directly compare against \textsc{JustGRPO}~\citep{ni2026flexibility} and \textsc{ESPO}~\citep{ou2026espo_principled}, since these were the strongest closely related diffusion-LM post-training baselines with usable public checkpoints. This enables a controlled comparison in which all directly evaluated methods are run under the same decoding protocol and evaluation pipeline.
\begin{table*}[!t]
\centering
\small
\caption{Single-sample accuracy (\%) on reasoning and coding benchmarks. Results are reported for two maximum completion lengths, 256 and 512. Results marked with \(\dagger\) are quoted directly from the original papers; these entries are included for context rather than as fully controlled direct comparisons.}
\label{tab:appendix_baseline_results}
\setlength{\tabcolsep}{4pt}
\begin{tabular}{l cc cc cc cc c}
\toprule
& \multicolumn{2}{c}{GSM8K} & \multicolumn{2}{c}{MATH-500} & \multicolumn{2}{c}{HumanEval} & \multicolumn{2}{c}{MBPP} & \\
\cmidrule(lr){2-3} \cmidrule(lr){4-5} \cmidrule(lr){6-7} \cmidrule(lr){8-9}
Method & 256 & 512 & 256 & 512 & 256 & 512 & 256 & 512 & Avg. \\
\midrule

D1\(\dagger\)~\citep{zhao2025diffuGRPO}
& 81.1 & 82.1 & 38.6 & 40.2 & 32.9 & 37.8 & 44.7 & 42.8 & 50.03 \\
DMPO\(\dagger\)~\citep{dmpo}
& 82.41 & 85.22 & 38.20 & 42.80 & -- & -- & -- & -- & -- \\

\bottomrule
\end{tabular}
\end{table*}

For completeness, \cref{tab:appendix_baseline_results} also reports published results for \textsc{D1}~\citep{zhao2025diffuGRPO} and \textsc{DMPO}~\citep{dmpo}, copied from their respective papers. These methods are closely related to our setting, but their public checkpoints were not available at the time of our experiments. We therefore include their numbers only as contextual reference points, rather than as fully controlled comparisons. The table marks such quoted results with \(\dagger\).

\section{LLM-as-a-Judge Diversity Prompt}
\label{app:diversity_judge_prompt}

We use the following prompt for the LLM-as-a-judge diversity evaluation. 
The prompt is designed to isolate differences in solution strategy, rather than correctness or surface-level variation. 
For each comparison, \texttt{\{problem\}}, \texttt{\{set\_a\}}, and \texttt{\{set\_b\}} are replaced with the problem statement and the two sampled answer sets being compared.

\begin{tcolorbox}
\tiny
\textbf{\color{systemgray}LLM-as-a-Judge Prompt for Diversity Evaluation}\\[0.4em]

{\color{tealprompt}\textbf{[Judge Role]}}\\
{\color{tealprompt}You are an expert judge evaluating diversity among LLM-generated answers to the same problem.}\\
{\color{tealprompt}Your goal is to assess approach diversity, not correctness or answer quality.}\\[0.75em]

{\color{purpleprompt}\textbf{[Problem]}}\\
{\color{purpleprompt}\texttt{\{problem\}}}\\[0.75em]

{\color{purpleprompt}\textbf{[Set A]}}\\
{\color{purpleprompt}\texttt{\{set\_a\}}}\\[0.75em]

{\color{purpleprompt}\textbf{[Set B]}}\\
{\color{purpleprompt}\texttt{\{set\_b\}}}\\[0.75em]

{\color{tealprompt}\textbf{[Definition of Distinct Approach]}}\\
{\color{tealprompt}An answer represents a distinct approach only if it uses a recognizably different core method, such as a different algorithm, data structure, mathematical reduction, equation setup, case analysis, proof strategy, or reasoning path.}\\[0.75em]

{\color{reasoningcolor}\textbf{[Do Not Count as Diversity]}}\\
{\color{reasoningcolor}-- different wording, formatting, verbosity, or variable names}\\
{\color{reasoningcolor}-- different final answers from the same method}\\
{\color{reasoningcolor}-- minor bugs in the same method}\\
{\color{reasoningcolor}-- random guesses or incoherent reasoning}\\
{\color{reasoningcolor}-- extra explanation that does not change the core method}\\[0.75em]

{\color{tealprompt}\textbf{[Instructions]}}\\
{\color{tealprompt}1. Identify the distinct approach types in each set.}\\
{\color{tealprompt}2. Compare the number and substance of these approach types.}\\
{\color{tealprompt}3. Prefer TIE if the difference is mostly superficial or unclear.}\\[0.75em]

{\color{purpleprompt}\textbf{[Response Format]}}\\
{\color{purpleprompt}Return exactly this JSON:}\\
{\color{purpleprompt}\texttt{\{\{}}\\
{\color{purpleprompt}\texttt{\ \ "set\_a\_approaches": ["short names of distinct approaches"],}}\\
{\color{purpleprompt}\texttt{\ \ "set\_b\_approaches": ["short names of distinct approaches"],}}\\
{\color{purpleprompt}\texttt{\ \ "winner": "A" | "B" | "TIE",}}\\
{\color{purpleprompt}\texttt{\ \ "confidence": "low" | "medium" | "high",}}\\
{\color{purpleprompt}\texttt{\ \ "reason": "one concise sentence"}}\\
{\color{purpleprompt}\texttt{\}\}}}
\end{tcolorbox}

\section{Limitations and Broader Impact}
\label{app:limitations_impact}

\subsection{Limitations}
\label{app:limitations}

Our study has several limitations. First,~\ourmodel\ relies on a diffusion-compatible sequence-level surrogate rather than exact trajectory log-probabilities. While this surrogate makes the objective practical for diffusion language models, it remains only an approximation to the ideal likelihood term. As a result, the optimized objective may not perfectly match the intended reward-tilted trajectory distribution, and performance may depend on surrogate design choices such as the masking scheme and Monte Carlo estimation strategy.

Second, our current experiments use terminal binary rewards: exact-match verification for math and execution-based test passing for code. Such rewards are simple and robust, but they are also coarse. In particular, when all sampled completions for a prompt receive the same reward, the centered reward term provides little or no relative learning signal. This makes training sensitive to rollout diversity, group size, and task difficulty, and may limit learning efficiency on prompts where successes are extremely rare or nearly universal.

\subsection{Broader Impact}
\label{sec:impact}

This work studies post-training objectives for diffusion language models, with experiments on mathematical reasoning and code generation. A potential positive impact of this line of research is improved reliability and usefulness of language models in domains that benefit from verifiable reasoning, such as education, scientific assistance, and software development. In particular, methods that improve solution coverage under sampling may help models produce a broader set of valid answers or implementations, which can be useful in settings where multiple correct solutions exist.

At the same time, stronger reasoning and code-generation models also carry risks. Improvements in code synthesis could be misused to generate malicious scripts, automate parts of cyberattacks, or lower the barrier to producing harmful software. Similarly, more capable reasoning models can be used to generate misleading technical explanations, polished but incorrect solutions, or other content that appears trustworthy despite being wrong. Even when used as intended, errors in model output may still mislead users in high-stakes settings if generations are accepted without verification.

Our work is primarily methodological and is not tied to a deployed system. Any real-world deployment of more capable post-training methods should, therefore, be accompanied by careful evaluation, monitoring for misuse, and appropriate safeguards to reduce the risk of harmful or misleading outputs.

\section{Training Compute}
\label{app:compute}
The post-training experiments can be reproduced on any 8 card modern parallel computing unit (NPUs, TPUs, GPUs) with at least 80 GB memory.
Each math post-training experiment takes approximately \(10\) hours, and the code post-training experiment takes approximately \(18\) hours. 

\section{Dataset and Model Licenses}
\label{app:dataset_model_licenses}

We use publicly released datasets, benchmarks, and models under their respective licenses. 
GSM8K~\citep{cobbe2021gsm8k} is released under the MIT License. 
MATH~\citep{lightman2023math500} is released under the MIT License. 
HumanEval~\citep{chen2021humaneval} is released under the MIT License. 
MBPP~\citep{mbpp} is released under CC BY 4.0. 
AceCode-87K~\citep{AceCoder} is released under the MIT License. 
KodCode-Light-RL-10K~\citep{xu2025kodcode} is released under CC BY-NC 4.0 and is used only for non-commercial research. 
Minerva Math~\citep{minerva} is listed under the MIT License in the public release we use. 
LiveCodeBench~\citep{jain2025livecodebench} release v5 is released under the Creative Commons license family and is used only for evaluation. 
LLaDA-8B-Instruct~\citep{zhu2025lladaInstruct} is released under the MIT License. 
We cite all original dataset and model sources and follow their usage protocols.

\end{document}